\documentclass[10pt,twocolumn,letterpaper]{article}

\usepackage{iccv}
\usepackage{times}
\usepackage{epsfig}
\usepackage{graphicx}
\usepackage{amsmath,wrapfig, graphicx}
\usepackage{amssymb,amsthm}
\usepackage{paralist}

\DeclareMathOperator*{\argmin}{arg\,min}
\newtheorem{proposition}{Proposition}

\usepackage{multirow}
\usepackage{colortbl, array}
\usepackage{hhline}
\usepackage[svgnames]{xcolor}

\usepackage[utf8]{inputenc} 
\usepackage[T1]{fontenc}    
\usepackage{amsmath,amsthm,amssymb,comment,tabularx}       
\usepackage{url}            
\usepackage{booktabs}       
\usepackage{amsfonts}       
\usepackage{nicefrac}       
\usepackage{microtype}      
\usepackage[linesnumbered,ruled]{algorithm2e}
\usepackage{graphicx,picins,caption,paralist,color,wrapfig}

\newcommand\numberthis{\addtocounter{equation}{1}\tag{\theequation}}
\usepackage{algpseudocode}
\usepackage{threeparttable}

\usepackage[pagebackref=true,breaklinks=true,letterpaper=true,colorlinks,bookmarks=false]{hyperref}

\newcommand{\myrowcolour}{\rowcolor[gray]{0.925}}
\newcommand{\highest}[1]{\textcolor{blue}{\mathbf{#1}}}

\definecolor{rulecolor}{RGB}{70,10,171}
\definecolor{tableheadcolor}{RGB}{120,50,200}

\newcommand{\topline}{ %
        \arrayrulecolor{rulecolor}\specialrule{0.1em}{\abovetopsep}{0pt}}%
\newcommand{\midtopline}{ %
        \arrayrulecolor{rulecolor}\specialrule{\lightrulewidth}{0pt}{0pt}}%
\newcommand{\bottomline}{ %
        \arrayrulecolor{rulecolor} \specialrule{\lightrulewidth}{0pt}{0pt}}%

\usepackage[pagebackref=true,breaklinks=true,bookmarks=false]{hyperref}

 \iccvfinalcopy 

\def\iccvPaperID{376} 

\ificcvfinal\fi
\begin{document}

\title{Dilated Convolutional Neural Networks for Sequential Manifold-valued Data}

\author{Xingjian Zhen$^{\dag\ast}$ \ \ Rudrasis Chakraborty$^{\ddag\ast}$ \ \ Nicholas Vogt$^\dag$ \ \ Barbara B. Bendlin$^\dag$ \ \ Vikas Singh$^\dag$
\\ $^\dag$University of Wisconsin Madison $\quad$ $^\ddag$University of California, Berkeley \\{\small $^\ast$Equal contribution}
}

\maketitle

\begin{abstract}
Efforts are underway to study ways via which the power
of deep neural networks can be extended to non-standard data types such as structured data (e.g., 
graphs) or manifold-valued data (e.g., unit vectors or special matrices).
Often, sizable empirical improvements are possible when the geometry of such
data spaces are incorporated into the design of the model, architecture, and the algorithms.
Motivated by neuroimaging applications,
we study formulations where the data are {\em sequential manifold-valued
measurements}. This case is common in brain imaging, where the samples correspond to
symmetric positive definite matrices or orientation distribution functions.
Instead of a recurrent model which poses computational/technical issues,
and inspired by recent results showing the viability of dilated convolutional models for
sequence prediction, we develop a dilated convolutional neural network architecture
for this task. On the technical side, we show how the modules needed in our network
can be derived while explicitly taking the Riemannian manifold structure into account.
We show how the operations needed can leverage known results
for calculating the weighted Fr\'{e}chet Mean (wFM).
Finally, we present scientific results for group difference analysis in Alzheimer's disease (AD)
where the groups are derived using AD pathology load: here the model finds
several brain fiber bundles that are related to
AD even when the subjects are all still cognitively healthy. 
\end{abstract}
\vspace*{-0.8cm}
\section{Introduction}\label{intro}
\vspace*{-0.3cm}
The classical definition of convolution assumes that the data are scalar or vector-valued and lie on discrete equally spaced intervals.
This assumption is ideal for natural images and central to how we use convolutional filters in deep neural networks but 
is far less appropriate for other domains where the data are structured such as
meshes, graphs or measurements on a manifold. In computer vision and machine learning, these problems that need deep learning models for structured data 
are studied under the topic called geometric deep learning \cite{bronstein2017geometric}, which has led to a number of elegant approaches including
convolutional neural networks (CNN) on non-Euclidean data \cite{cohen2018spherical,kondor2018covariant}. 
The reason this is important is that mathematically, non-Euclidean data violates a number of key properties of Euclidean spaces such as 
a global linear structure and coordinate system, as well as 
shift invariance/equivariance. As a result,
the core operations we use in classical statistics and machine learning as well as within deep neural network architectures
often need to be tailored based on the geometry and specifics of the data at hand.
When such adjustments are made in modern deep learning architectures, a number of authors have reported sizable improvements in the performance of the learning algorithms 
\cite{chakraborty2018manifoldnet, chakraborty2018h,kondor2018generalization,cohen2018spherical,huang2016building,huang2017riemannian,cohen2016steerable}.

We should note that specializing learning methods to better respect or exploit the {structure} (or geometry) of the data are not a new development.
Time series data are common in finance \cite{tsay2005analysis}, and as a result, has been
analyzed using specialized methods in statistics for decades. 
Surface normal vectors on the unit sphere have been widely used in graphics \cite{straub2015dirichlet}, and
probability density functions, as well as covariance matrices, are common in both machine learning and computer vision \cite{srivastava2007riemannian,dominici2002use}. 
In neuroimaging, which is a key focus of our paper, the structured measurement at a voxel of an image may capture water diffusion \cite{basser1994mr,wang2005dti,Lenglet2006,Jian_NI07,aganj2009odf,cheng2012efficient} or local structural change \cite{hua2008tensor,zacur2014multivariate,kim2017riemannian}. The latter example
is commonly known as the Cauchy deformation tensor (CDT) \cite{kim2017riemannian} and has been utilized
to achieve improvements over brain imaging methods such as tensor-based morphometry \cite{lepore2006multivariate,ridgway2009statistical,ashburner2004morphometry}. 
When the mathematical properties of such data are exploited,
one often needs new loss functions and specialized optimization schemes. This step often involves
first defining an
intrinsic metric for the underlying  geometry (structure) of the data.
It is important to note that within
geometric deep learning for {\em manifolds}, two types of settings are often considered. 
The {\bf first} case is where the data are functions on a manifold. The {\bf second} case corresponds to the setting where
data are sample points on a manifold, such as a Riemannian manifold. In this paper, we study the second setting, which is not covered in the form described
here in existing works including \cite{bronstein2017geometric}. 

When the structure or geometry of the data informs the formulation of the learning task (or algorithm), 
we obtain differential geometry inspired algorithms where the role of the extrinsic or intrinsic
metric induced by the data is explicit. 
Many datasets do {\em not} have a temporal or sequential component associated with each sample. 
However, the analysis of temporal (or sequential) data is an important area of machine learning and vision, e.g., 
within action recognition \cite{afsari2012group,bissacco2001recognition,turaga2008statistical} and
video segmentation \cite{he2012incremental}, 
the study of analogous geometric ideas in this regime, especially within deep learning, is limited.
Specifically, there are few existing proposals describing deep neural network models for structured (or manifold-valued)
{\em sequential} data. 
Recently in \cite{chakraborty2018statistical}, the authors proposed a {\em recurrent} model for the manifold of symmetric positive definite (SPD) matrices.
This work is interesting and replaces a number of blocks
within a recurrent model with the ``statistical recurrent units''. 
But it is known that training recurrent models is
more involved than convolutional architectures -- shortly, our experiments will show that a $2\times$ speed-up
(by using a convolutional instead of a recurrent model) can be achieved.
While
the current consensus, within the community, is that sequential data should involve a recurrent network \cite{elman1990finding},
as noted by \cite{bai2018convolutional}, emerging
results indicate that convolutional architectures often perform superior to recurrent networks on ``sequential'' applications such as audio synthesis.
In fact, even historically, convolutional models were used for
1-D {\em sequential} data \cite{hinton1990connectionist,lecun1995convolutional}. Now, given that
most use-cases of learning sequential models on manifold-valued data will {\em not} require the infinite memory capabilities offered by a recurrent model,
it seems natural to investigate the extent to which convolutional models may suffice.
Notice that in order to get the long effective memory from a CNN model, one needs to increase the depth and/or increase the receptive field: 
this is provided by extensions such as dilated convolutions. 
We find that the two key ingredients in \cite{bai2018convolutional} to achieve similar or better performance than a recurrent model for sequential tasks involves \begin{inparaenum}[\bfseries (a)] \item using dilations to increase the receptive field of each convolution and 
\item using residual connections to design a deeper but stable network. \end{inparaenum}
It seems logical
that these developments should be an ideal starting point in
designing models and algorithms for {\bf sequential manifold-valued data} -- the goal of this work. 
Our key {\bf contribution} is the design of a Dilated CNN model for sequential manifold-valued data and showing its applicability in performing statistical analysis of brain images, specifically,
diffusion-weighted MR images. To do so, we 
\begin{inparaenum}[\bfseries (a)] \item define dilation for the convolution operator on the manifold of interest \item define residual connections for our architecture
\item define weight normalization/dropout to add regularization/stability for the deeper network. \end{inparaenum}
We show that this yields an efficient formulation for sequential manifold-valued data, where few
exist in the literature at this time. On the scientific side,
we show that such a construction gives us the ability to identify structural connectivity changes in
asymptomatic individuals who are at risk for developing
Alzheimer's disease (AD) but are otherwise cognitively healthy. 

\vspace*{-0.3cm}
\section{Preliminaries}
\label{sec:prelim}
\vspace*{-0.2cm}
The motivation of this work is the analysis of sequential manifold-valued data, using deep architectures.
As described above, our architecture utilizes ideas presented earlier in the context of 
dilated convolutional neural networks (DCNN) on Euclidean spaces \cite{bai2018convolutional}. To set up our formulation, we review the standard
DCNN formulation and then describe our proposed manifold-valued DCNN framework.

\textbf{Dilated Convolutions \cite{bai2018convolutional}:} Given a 1-D input sequence $\mathbf{x}: \mathbf{N} \rightarrow \mathbf{R}^n$ and a  kernel $w: \left\{0, \cdots, k-1\right\} \rightarrow \mathbf{R}$, the dilated convolution function $\left(\mathbf{x}\star_d w\right): \mathbf{N} \rightarrow \mathbf{R}^n$ is:
\vspace*{-0.35cm}
\begin{align}
\label{theory:eq1}
\left(\mathbf{x}\star_d w\right)(s) = \sum_{i=0}^{k-1} w(i) \mathbf{x}(s-id),
\end{align}
\vspace*{-0.5cm}

where $\mathbf{N}$ is the set of natural numbers, and $k$ and $d$ are the kernel size and the dilation factor respectively. 
Notice that with $d=1$, we get the normal convolution operator. In a dilated CNN, the receptive field size will depend on the depth of the network as well as on the choice of $k$ and $d$. Thus, the authors in \cite{bai2018convolutional} suggested the use of {\it residual connections} \cite{he2016deep} -- this was found to provide stability for deeper networks.
Notice that, unlike the standard residual network connection, here the authors used a $1\times1$ convolution layer in order to match the width of the input and the output. Additionally,
in order to regularize the network, the authors used {\it weight normalization} \cite{salimans2016weight} and {\it dropout} \cite{srivastava2014dropout}.
The weight normalization was applied to the kernel of the dilated convolution layer. The dropout was
implemented by randomly zeroing out an entire output channel of a dilated convolution layer. Finally, as an activation function, the authors used ReLU non-linearity. A schematic diagram of a standard dilated CNN is given in Fig. \ref{fig0}.

\begin{figure}[!t]
\setlength{\abovecaptionskip}{-0cm}
\setlength{\belowcaptionskip}{-0.65cm} 
        \centering
                \includegraphics[scale=0.6]{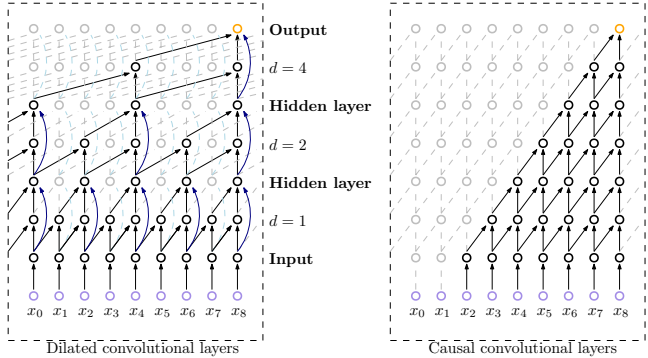}
               \caption{\footnotesize Schematic diagram of dilated CNN and causal CNN (see \cite{bai2018convolutional} for definition and additional description).}\label{fig0}

\end{figure}

Next, we discuss generalizing the operations needed within a DCNN so that they can operate on manifold-valued data. Specifically, we will generalize the following operations: \begin{inparaenum}[\bfseries (1)] \item Dilated convolution \item Residual connection \item Weight Normalization \item ReLU and \item Dropout, \end{inparaenum} to the setting
where data are manifold-valued. 

Recently in \cite{chakraborty2018manifoldnet}, the authors proposed a CNN architecture for manifolds and/or manifold-valued data. We can utilize some of these ideas towards
deriving the dilated convolution operation. 
Before discussing the details of the definition of dilated CNN for manifold-valued data, we will first introduce some notations, concepts, and terminology. 
\vspace*{0.1cm}

%
{\bf Assumptions:} We use $(\mathcal{M}, g)$ to denote a Riemannian manifold $\mathcal{M}$ with the Riemannian metric $g$ and $d_{\mathcal{M}}:\mathcal{M}\times \mathcal{M}\rightarrow [0, \infty)$ denotes the distance induced by the metric $g$. We assume that the samples on $\mathcal{M}$ lie inside a regular geodesic ball of radius $r$ centered at $p$, $\mathcal{B}_r(p)$, for some $p \in \mathcal{M}$ and $r = \min\left\{r_{\text{cvx}}\left(\mathcal{M}\right), r_{\text{inj}}\left(\mathcal{M}\right)\right\}$. Here, $r_{\text{cvx}}$ and $r_{\text{inj}}$ are the convexity and injectivity radius of $\mathcal{M}$ \cite{groisser2004newton}.

  {\bf Weighted Fr\'{e}chet mean (wFM):} Let $\left\{X_i\right\}_{i=1}^N$ be samples on $\mathcal{M}$. The authors in \cite{chakraborty2018manifoldnet} define the convolution operation using the
  weighted Fr\'{e}chet mean (wFM) \cite{Frechet1948a} of $\left\{X_i\right\}$. Consider a
  one dimensional kernel $\left\{w(i)\right\}_{i=1}^N$ satisfying the convexity constraint,
  i.e.,  \begin{inparaenum}[\bfseries (a)] \item $\forall i, w(i) > 0$ \item $\sum_i w(i) = 1$. \end{inparaenum}
  Then, the wFM (uniqueness is guaranteed by the statement above) is defined as:
\vspace*{-0.3cm}
\begin{equation}
\label{theory:eq2}
\textsf{wFM}\left(\left\{X_i\right\}, \left\{w\right\}\right) = \argmin_M \sum_{i=1}^N w(i) d_{\mathcal{M}}^2(X_i, M),
\end{equation}
\vspace*{-0.4cm}

{\bf Group of isometries:} The set $I(\mathcal{M})$ of all isometries of $\mathcal{M}$ forms a group
  with respect to function composition. We will use $G$ to denote this group and for $g
  \in G$, and $X \in \mathcal{M}$, let $g.X$ denote the result of
  applying the isometry $g$ to point $X$ (`.' simply denotes the group action).  
  \vspace*{0.1cm}

  {\bf Key Application focus:}
  Diffusion-weighted imaging (DWI) is a magnetic resonance imaging (MRI) technique that
  measures
  the diffusion of water molecules to generate contrast in MRI, and
  has been widely applied to measure the loss of structural connectivity in the brain. 
  At each voxel in the image, water diffusion can be variously represented: two common options are
  using an elliptical approximation (see Fig. \ref{DTI}(a)) where a $3 \times 3$ covariance matrix expresses the diffusivity properties or an orientation
  distribution function where one represents the probability densities of water diffusion over different orientations.
  One can divide the 3D image into anatomically meaningful parcels in Fig. \ref{DTI}(b) and then run standard tractography routines
  to estimate the strength of connectivity between each pair of anatomical parcels \cite{raamana2018graynet}.
  The fiber bundles, hence estimated, are shown in Fig. \ref{DTI}(c). For analysis, one often focuses
  on certain important fiber bundles instead of analyzing the full set of fibers.
  Notice that if we specify a starting and ending anatomical region for a fiber bundle,
  we can consider the corresponding covariance matrices encountered on this ``path'' as multi-variate
  manifold-valued measurements of this function. This is precisely the type of sequential
  manifold-valued data that we will seek to model in this paper. 

\begin{figure}[!ht]
\setlength{\abovecaptionskip}{-0cm}
\setlength{\belowcaptionskip}{-0.65cm} 
        \centering
                \includegraphics[scale=0.04]{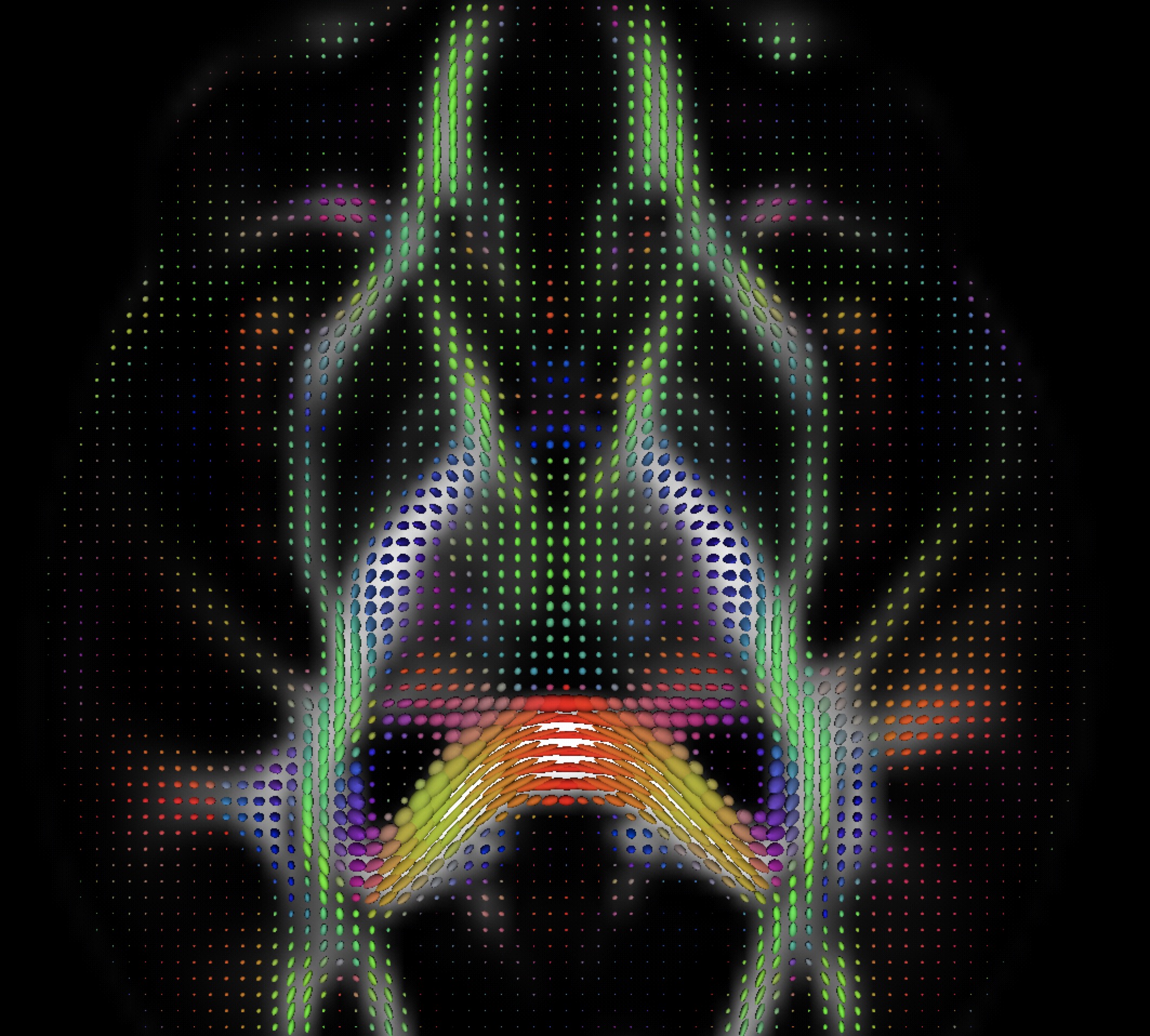}
                \includegraphics[scale=0.195]{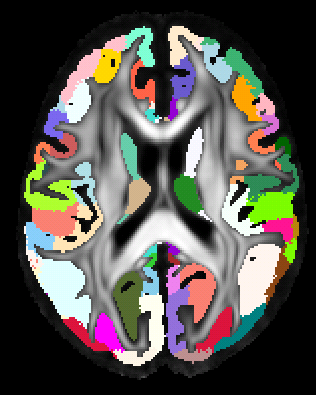}
                \includegraphics[scale=0.25]{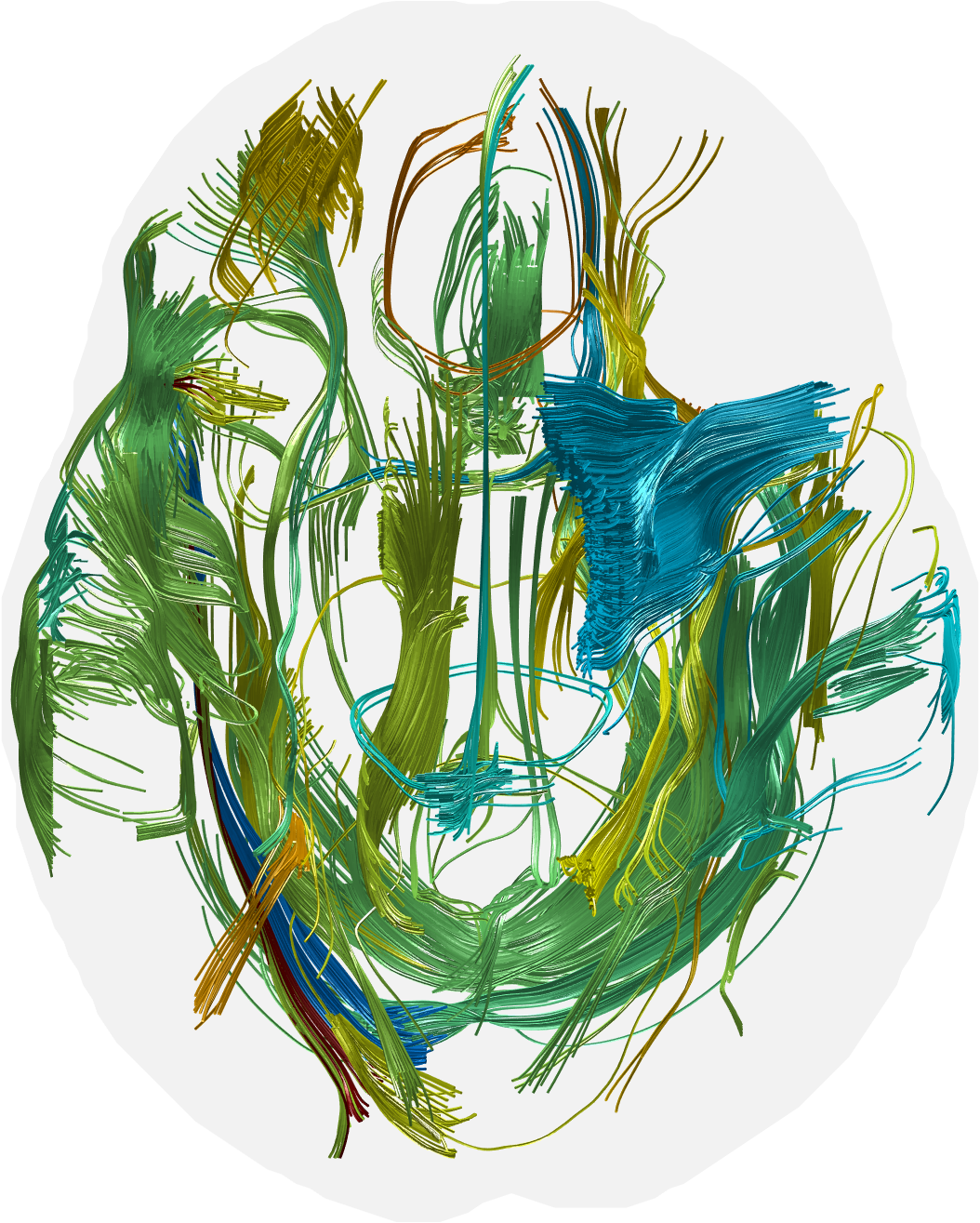}
               \caption{\footnotesize {\it (Left-Right)} (a) diffusion MRI, (b) Parcels, (c) Fiber bundles}\label{DTI}
\end{figure}{}

\section{Dilated convolutions for manifold-valued measurements}\label{theory}
We now describe how to obtain the specific components needed in our architecture for manifold-valued data. 

\textbf{Dilated convolution operator:} Given a 1-D input sequence $X: \mathbf{N} \rightarrow \mathcal{M}$ and a  kernel $w: \left\{0, \cdots, k-1\right\} \rightarrow \mathbf{R}$ satisfying the convexity constraint, the dilated convolution function $\left(X\star_d w\right): \mathbf{N} \rightarrow \mathcal{M}$ is defined as:
\vspace*{-0.3cm}
\begin{equation}
\label{theory:eq3}
\left(X\star_d w\right)(s) = \argmin_M \sum_{i=0}^{k-1} w(i) d_{\mathcal{M}}^2(X(s-id), M),
\end{equation}
\vspace*{-0.6cm}

where as before, $k$ and $d$ are the kernel size and dilation factor respectively. Observe that the convexity constraint on the kernel is merely to ensure that the result also
lies on the manifold.
We will use the weighted Fr\'{e}chet mean (wFM) as a dilated convolution operator. This choice is mathematically justified because 
\begin{inparaenum}[\bfseries (1)] \item Eq. \eqref{theory:eq1} is the minimizer of the weighted variance which is wFM, if the choice of distance is the $\ell_2$ distance. \item We will
  show in Proposition \eqref{theory:prop1} that the dilated convolution operator is equivariant to the action of $G$. \end{inparaenum}
This is a direct analog of its Euclidean counterpart. Notice that the dilated convolution operator defined in \eqref{theory:eq1} is equivariant to translations, i.e.,
if $\mathbf{x}$ is translated by some amount $\mathbf{t}$, so is the result $\left(\mathbf{x}\star_d w\right)$.
On the manifold $\mathcal{M}$, the analog of translation is the action of $G$, hence
the equivariance of $\left(X\star_d w\right)$ with respect to $G$
is a desirable property. 
\vspace*{-0.2cm}
\begin{proposition}
\label{theory:prop1}
Using notations in \eqref{theory:eq3} and given $w$ satisfying the convexity constraint, let $F: X \mapsto
\left(X\star_d w\right)$. Then, $F$ is $G$-equivariant, i.e., $F$ is equivariant to the action of $G$.
\end{proposition} 
\begin{proof}
  Observe that, if $g \in G$ acts on $X$, then, $X(s-id)\mapsto g.X(s-id)$, for all $s, d, i$. Since $g$ is an element of isometry group, therefore,
  $d_{\mathcal{M}}(g.X(s-id), g.M) = d_{\mathcal{M}}(X(s-id), M)$, for all $M \in \mathcal{M}$. So, $g.M = \left(g.X\star_d w\right)(s)$ iff $M=\left(X\star_d w\right)(s)$, which concludes our proof.
\end{proof}  
\vspace*{-0.2cm}
In \eqref{theory:eq3}, since $\left(X\star_d w\right)$ is a $\mathcal{M}$ valued function, we will use $M$ as a manifold-valued function, i.e., $M(s) =   \left(X\star_d w\right)(s)$. Similar to the Euclidean dilated convolution layer, we learn multiple dilated kernels (given by the number of output channels) for a dilated convolutional layer. 
\vspace*{-0.cm}

\textbf{Residual connection:} Let $X$ and $F$ be the input and output of a dilated convolutional layer where the numbers of channels are $c_{in}$ and $c_{out}$. Then, analogous to the Euclidean residual connection, we define the residual connection using two steps: \begin{inparaenum}[\bfseries (a)] \item First, concatenate $X$ and $F(X)$ to get $(c_{in}+c_{out})$ number of channels. \item Use wFM to extract $c_{out}$ number of outputs. \end{inparaenum} More formally, let $R(X, F_X)$ be the output of the residual connection, then the $k^{th}$ channel of the residual connection, $R_k(X, F(X))$ is given by:
\vspace*{-0.4cm}
{\small
\begin{align*}
\label{theory:eq4}
&R_k(X, F(X))(s) \overset{def}{\equiv} \argmin_M \\[-2.5mm]
&\left(\sum_{i=1}^{c_{in}}w_{k}(i) d_{\mathcal{M}}^2(X_i(s), M) + \sum_{j=1}^{c_{out}}w_{k}(j+c_{in})d_{\mathcal{M}}^2(F_j(s), M)\right),\\[-1.5mm]
&\mbox{s.t.}\sum \nolimits_i w_{k}(i) = 1,\forall w_{k}(i)>0,\numberthis
\end{align*}
}
\vspace*{-0.5cm}

where, $k \in \left\{1, \cdots, c_{out}\right\}$ and $X_i$ and $F_j$ denotes the $i^{th}$ and $j^{th}$ channel of $X$ and $F$ respectively. 
\vspace*{0.1cm}

\textbf{Weight normalization, ReLU, and Dropout:} The weight normalization in the standard Euclidean convolutional network is not needed here since
we impose
a convexity constraint on the kernel. We argue that since
Dropout is a regularizer, we will not use dropout for our manifold-valued
DCNN implementation because of the implicit regularization due to
the convexity constraint. As argued in \cite{chakraborty2018manifoldnet}, wFM is both \begin{inparaenum}[\bfseries (a)] \item a contraction mapping \cite{chakraborty2018manifoldnet} and
\item a nonlinear mapping \end{inparaenum} and hence ReLU or any other non-linearity is not
strictly necessary.
Here, similar reasoning explains why a ReLU is not needed (since the contraction
and non-linear mapping are provided directly by wFM).
\vspace*{0.1cm}

\textbf{Equivariance and Invariance:} A few reasons why convolutional networks are so powerful are \begin{inparaenum}[\bfseries (a)]
\item translational equivariance of a convolution layer and so, weights can be shared across an image
\item translational invariance property of the entire convolutional network which is the property of the fully connected last layer.
\end{inparaenum}
As we showed above, the way we defined our dilated convolution operator leads to equivariance to the action of $G$.
But we still have not shown that the {\em last layer} can be designed in a way that
the output of the network does not change with respect to the action of $G$.
So, we still need an analogous $G$-invariant last layer.
\vspace*{0.1cm}

\textbf{Invariant last layer:} Analogous to the Euclidean recurrent model/ dilated CNN, in the last layer we will only consider the
output of the last time point of a sequence, i.e., if $X$ is the output of the last dilated convolutional layer with $c$ number of channels,
then the input of our last layer is $\left\{X_i(N)\right\}_{i=1}^c$, where $X(N) \in \mathcal{M}$ is the value of the last time point.
We know already that $\left\{X_i(N)\right\}$ are $G$-equivariant. So, in order to make the entire dilated convolutional network $G$ invariant,
we need an invariant last layer. This is analogous to the translational invariant property of a fully connected (FC) layer in the traditional (Euclidean) dilated CNN.
We design our last invariant layer as follows: \begin{inparaenum}[\bfseries (a)] \item We will first learn $nC$ number of wFM (let denoted by $\left\{\mu_i\right\}_{i=1}^{nC}$)
  of $\left\{X_i(N)\right\}_{i=1}^c$ using \eqref{theory:eq2}, where $nC$ is a hyperparameter.
\item For all $i \in \left\{1, \cdots, c\right\}$, and for all $j \in \left\{1, \cdots, nC\right\}$, we compute the distance between $X_i(N)$ and $\mu_j$, denoted by $d_{ij}$.
\item Thus, for each $X_i(N)$, we get $nC$ number of feature representations.
\item We will use a standard fully connected (FC) layer with $c\times nC$ features as input and the desired number of outputs.
\end{inparaenum}

\vspace*{-0.3cm}
\begin{proposition}
\label{theory:prop2}
The last layer is $G$-invariant.
\end{proposition}
\vspace*{-0.4cm}
\begin{proof}
Observe that $d_{ij} = d_{\mathcal{M}}\left(X_i(N), \mu_j\right)$. From Proposition \ref{theory:prop1}, we know that $\mu_j$ is $G$-equivariant, hence, $\mu_j \mapsto g.\mu_j$, for some $g \in G$ if $\forall i, X_i(N) \mapsto g.X_i(N)$. But, $d_{\mathcal{M}}\left(X_i(N), \mu_j\right) = d_{\mathcal{M}}\left(g.X_i(N), g.\mu_j\right)$, which concludes the proof.
\end{proof}
\vspace*{-0.2cm}
In order to reduce the number of parameters in the last layer, we propose a parameter efficient last layer which is defined as using a FC layer on the tangent space, i.e., input $\left\{\textsf{Log}\left(X_i(N)\right)\right\}_{i=1}^c$ as input to the FC layer, where $\textsf{Log}$ is the Riemannian inverse exponential map.

Now, we have all components of our dilated CNN on manifold-valued data.
A schematic of our model is shown in Fig. \ref{fig1}. The building block for a 2-layer manifold
DCNN is shown in Alg. \ref{alg:dcnn}. Note that the network
parameters are scalar-valued, with a convexity constraint.
In order to enforce
the convexity constraint, i.e., $\left\{w(i)\right\} \geq 0$ and $\sum_i w(i) = 1$, we will learn $\left\{\sqrt{w(i)}\right\}$, which can be any real value. We will enforce the sum constraint by normalization. Thus we will use SGD to learn $\left\{\sqrt{w(i)}\right\}$.  

\begin{figure}[!b]
\setlength{\abovecaptionskip}{-0cm}
\setlength{\belowcaptionskip}{-0cm} 
        \centering
                \includegraphics[scale=0.68]{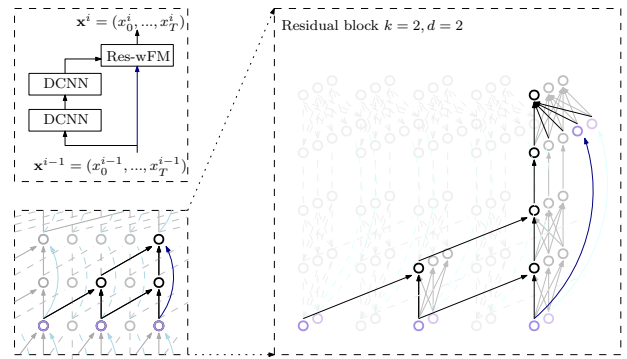}
               \caption{\footnotesize Schematic diagram of the residual block of manifold DCNN. There're two DCNN blocks and one residual connection in one block. wFM is used to extract the $c_{out}=3$ channels from the concatenation.}\label{fig1}
\end{figure}

{\begin{algorithm}[!t]
                \caption{ \label{alg:dcnn} A basic $i^{th}$ DCNN building block with two convolution layers}
                \begin{algorithmic}
                    \Function {DCNN variables}{$N, c^1_{in}, c^1_{out}, c^2_{out}, c_{res}, k_1, d_1, k_2, d_2, nC,c$}
                        \State $x^{i-1} \leftarrow \rm{Input}(c^1_{in}, N)$
                        \State $y_1 \quad \leftarrow \rm{Dilated\_Conv}(x^{i-1}, c^1_{in}, c^1_{out}, k_1, d_1)$
                        \State $y_1 \quad \leftarrow \rm{Dilated\_Conv}(y_1, c^1_{out}, c^2_{out}, k_2, d_2)$\
                        \State $x^{i} \quad \leftarrow \rm{Residual}(x^{i-1}, y_1, c^1_{in}, c^2_{out}, c_{res})$
                        \State $y_o \quad \leftarrow \rm{Inv}(x^{i}, nC, c)$ (For last DCNN block)
                        \EndFunction
                \end{algorithmic}
\end{algorithm}}

\section{Experiments}\label{results}
In this section, we apply the manifold DCNN to answer the following questions: 
\begin{inparaenum}[\bfseries (1)] \item By replacing a RNN with our DCNN with a manifold constraint, 
what improvement in terms of the number of parameters/time can we achieve, 
without sacrificing performance? \item For computer vision applications, 
how much improvement can we get? 
\item When using our method for scientific analysis of neuroimaging data, can we
obtain promising results that show that such models can enable discoveries
beyond current capabilities? \end{inparaenum}

Next, we will answer the questions above by analyzing the comparative performance of manifold DCNN via
four experiments: 
\begin{inparaenum}[\bfseries (1)] 
\item two computer vision applications of classifying videos and 
\item two neuroimaging experiments for scientific discoveries related to Alzheimer's disease.
\end{inparaenum}

\vspace{-0.2cm}
\subsection{Improvement in terms of parameters/time on synthetic and real computer vision datasets}
\vspace{-0.1cm}
In this section, we organize two sets of experiments:
\begin{inparaenum}[\bfseries (1)]
\item Classification of different moving patterns on the Moving MNIST data
\item Classification of $11$ actions on the UCF-11 data.
\end{inparaenum}
Both these experiments serve as empirical evidence of the efficiency of manifold DCNN in terms of the number of parameters and time per epoch. We compared our method with five state-of-the-art sequential models: SPD-SRU \cite{chakraborty2018statistical}, 
LSTM \cite{hochreiter1997long}, SRU \cite{oliva2017statistical}, TT-GRU and TT-LSTM \cite{yang2017tensor}. 
For all methods except TT-GRU and TT-LSTM, before the sequence process module, 
we used a convolution block. For manifold DCNN and SPD-SRU (also for manifold-valued data), 
between the convolution block and the sequence process unit, 
we include a covariance block analogous to \cite{yu2017second}. The architecture of this experiment is shown in Fig.~\ref{fig2}.

\begin{figure}[!t]
\setlength{\abovecaptionskip}{-0.cm}
\setlength{\belowcaptionskip}{-0.3cm} 
        \centering
                \includegraphics[scale=0.65]{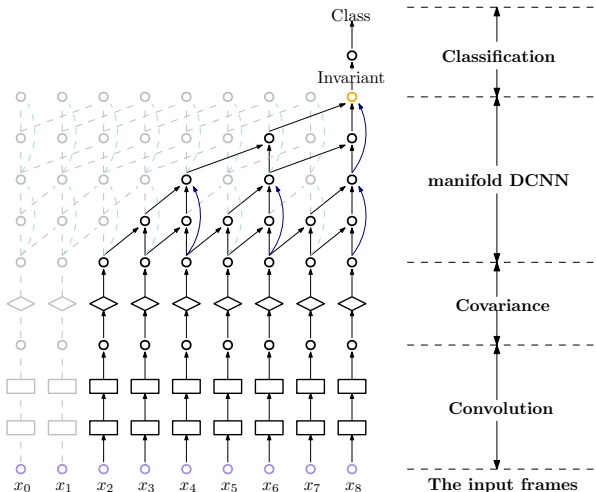}
               \caption{\footnotesize Schematic diagram of the network architecture for vision datasets. We use two CNN to extract the features. And we calculate the covariance between feature channels to get the SPD matrices. In the last layer, we use G-invariant and a fully connected layer to do the classification.}\label{fig2}
\vspace*{-0.2cm}
\end{figure}

As one of the key operations of DCNN is $\textsf{wFM}$, below we will
use an efficient recursive provably consistent estimator of $\textsf{wFM}$ on the space of covariance matrices (SPD with some added small noise along diagonal). Let $X(s)$ be an SPD matrix for all $s \in \mathbf{N}$, and then the $n^{th}$ recursive $\textsf{wFM}$ estimator, $M_n$ is given as:
\vspace*{-0.5cm}
\begin{align}
\label{rodf}
    M_0 = X(s) \qquad M_n = \Gamma_{M_{n-1}}^{X(s-n*d)}\left(\frac{w(n)}{\sum_{j=0}^{n}w(j)}\right),
\end{align}
\vspace*{-0.cm}
where $\Gamma$ is the shortest geodesic on the manifold of SPD matrices equipped with the canonical affine invariant Riemannian metric \cite{ho2013recursive}. 

\vspace*{-0.2cm}
\vspace{-5pt}
\subsubsection{Moving MNIST: Moving pattern classification}
\vspace{-5pt}

We
generated the Moving MNIST data according to the algorithm proposed in \cite{srivastava2015unsupervised}. In this experiment, we classify the moving patterns of different digits. For each moving pattern, we generated $1000$ sequences with length $20$ showing $2$ digits 
moving in the same pattern in a $64\times 64$ frame. The moving speed and the direction 
are fixed inside each class, but the digits are chosen randomly. 
In this experiment, the difference in the moving angle from two sequences across different classes is at least $5^\circ$. 

\setlength{\intextsep}{2pt}%
\setlength{\columnsep}{5pt}%
\begin{table}[!b]
\vspace*{-0.4cm}
\setlength{\abovecaptionskip}{-0cm}
\setlength{\belowcaptionskip}{-0.15cm} 
   \centering
   \scalebox{0.60}{
\begin{tabular}{cccccc} 
\topline\myrowcolour
 & & {\bf time (s)} & \multicolumn{3}{c}{{\bf Test acc.}} \\
\arrayrulecolor\hhline{---~~~}\arrayrulecolor{rulecolor}\hhline{~~~---}\myrowcolour
\multirow{-2}{*}{\bf Model} & \multirow{-2}{*}{\bf \# params.} & {\bf / epoch} & $30^\circ$versus $60^\circ$ & $10^\circ$versus $15^\circ$ & $10^\circ$versus $15^\circ$versus $20^\circ$ \\
\midtopline
DCNN & $\highest{1517}$ & $\sim 4.3$ & $\highest{1.00\pm 0.00}$ & $\highest{1.00 \pm 0.01}$ & $\highest{0.95 \pm 0.01}$ \\
SPD-SRU & 1559 & $\sim 6.2$ & $\highest{1.00\pm 0.00}$ & ${0.96 \pm 0.02}$ & ${0.94 \pm 0.02}$ \\
TT-GRU & $2240$ & $\sim\highest{2.0}$ & $\highest{1.00 \pm 0.00}$ & $0.52 \pm 0.04$ & $0.47 \pm 0.03$ \\
TT-LSTM & $2304$ & $\sim\highest{2.0}$ & $\highest{1.00 \pm 0.00}$ & $0.51 \pm 0.04$ & $0.37 \pm 0.02$ \\
SRU & $159862$ & $\sim 3.5$ & $\highest{1.00 \pm 0.00}$ & $0.75 \pm 0.19$ & $0.73 \pm 0.14$ \\
LSTM & $252342$ & $\sim 4.5$ & $0.97 \pm 0.01$ & $0.71 \pm 0.07$ & $0.57 \pm 0.13$ \\
\bottomline
\end{tabular}
}
\caption{\footnotesize Comparative results on Moving MNIST. Our model achieves the highest accuracy (in blue) with the least \# of parameters in all setups.}
\label{results:tab1}
\end{table}

\textbf{Results:} In Table~\ref{results:tab1}, the results show that our method not only achieves 
the best test accuracy with the smallest number of parameters but is
also $1.5$ times faster than the SPD-SRU  which has the second smallest \# of parameters.
The kernel of CNN we use has size $5\times 5$ with the input channel and
output channel set to $5$ and $10$ respectively.
All parameters are chosen in a way to use the
fewest number of parameters without deteriorating the test accuracy.

{\bf Scalability:} We assess
the running time (training and testing) of manifold DCNN with respect to the SPD matrix size. From Fig. \ref{fig2.5}(a), we can see that as the matrix size increases, the training time increases, while the testing time remains almost the same. This is a desirable property as it indicates that inference time does not depend on matrix size. Also, for different orientations differences, manifold DCNN gives almost perfect classification accuracy with very small standard deviation, as shown in Fig. \ref{fig2.5}(b). 
\begin{figure}[!t]
\setlength{\abovecaptionskip}{-0cm}
\setlength{\belowcaptionskip}{-0cm} 
        \centering
                \includegraphics[scale=0.109]{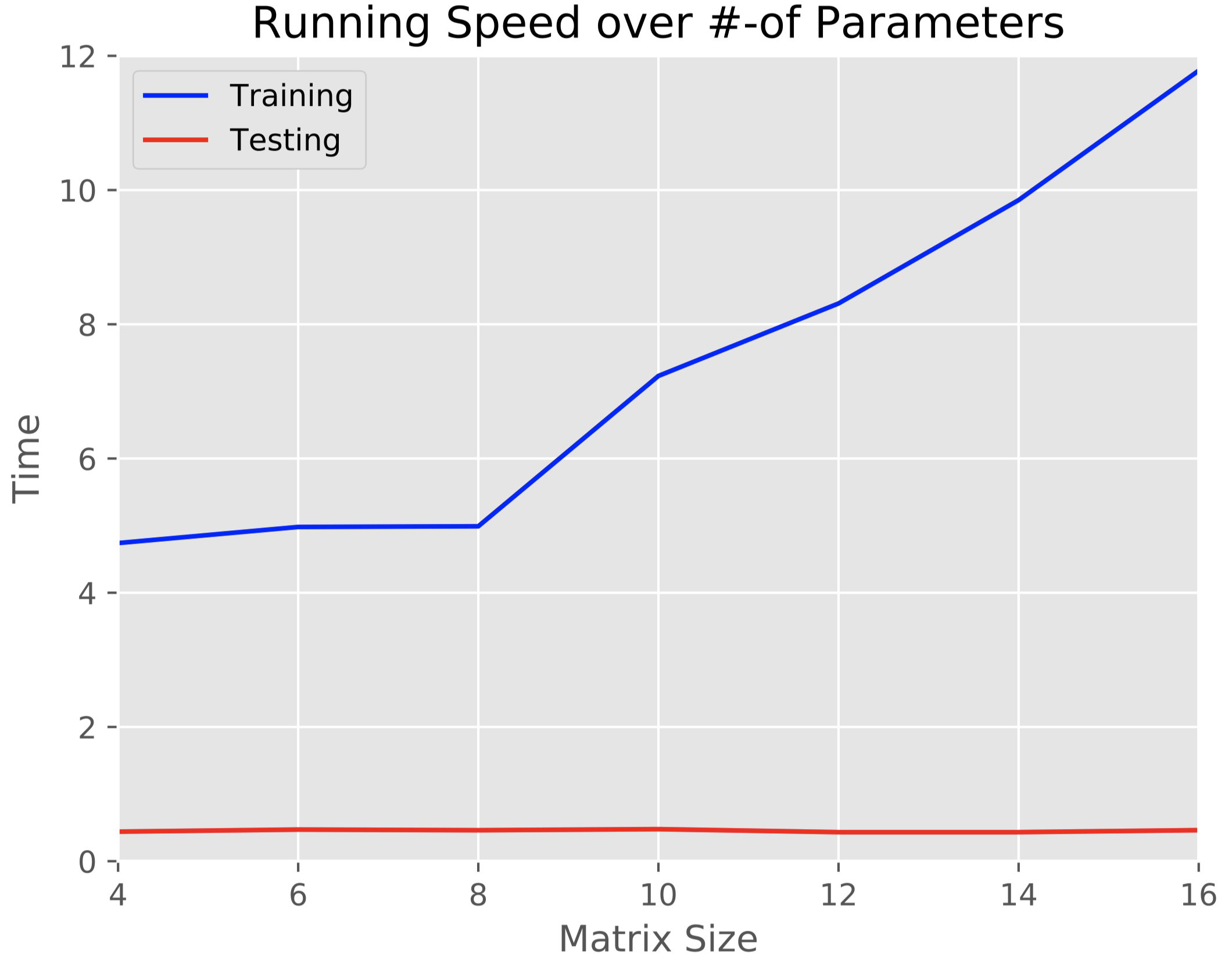}
                \includegraphics[scale=0.109]{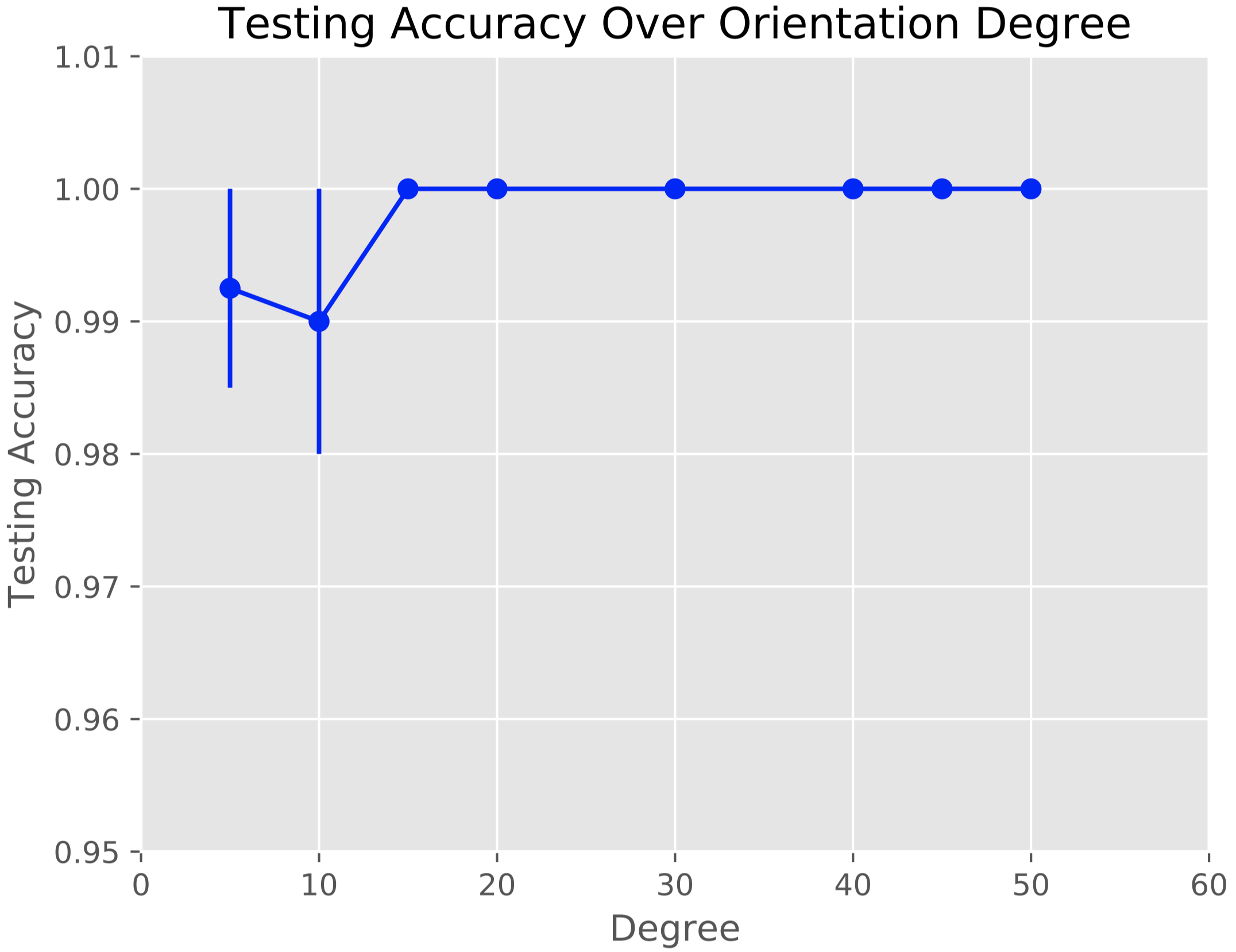}  
                \caption{\footnotesize Left: time versus matrix size. As the matrix size increases, the training time inevitably increases but the testing time consistently remains extremely small. Right: accuracy versus
                  degree difference of orientation in the dataset. Beyond the degree difference as small as $15^\circ$, the error bar becomes negligible implying our model quickly becomes very robust.
                  }\label{fig2.5}
\vspace*{-0.5cm}
\end{figure}

\vspace{-5pt}
\subsubsection{UCF-11: Action classification}
\vspace{-5pt}
The UCF-11 dataset \cite{liu2009recognizing} contains 1600 video clips of 11 different classes, such as basketball shooting, diving, etc.
The video lengths (frame sequences) vary from $204$ to $1492$, with the resolution of each frame being $320 \times 240$. We sample every $3$ frames, resize each frame to $160 \times 120$, and clip the frame sequences to have the length of $50$. For our method, we chose two convolution layers with kernels $7 \times 7$ and output channels $4$ and $6$ before the DCNN block. Hence, the dimension of the covariance matrices is $7\times 7$. 
For the manifold DCNN block, we use three residual blocks, 
with channels set to be $[1,3,3]$; $[3,3,4]$ and $[4,4,4]$ respectively. 
The kernel size is $5$ for each residual block with the initial dilation number being $1$ 
(if not specified, the initial dilated number is always $1$ in this paper.).
For TT-GRU and TT-LSTM, we follow the same setting as given in \cite{yang2017tensor}. 
For SPD-SRU, SRU, and LSTM, we use the same parameters as in \cite{chakraborty2018statistical}. 
All models achieve $>90\%$ training accuracy. 

\textbf{Results:}
Test accuracy with the number of parameters and time per epoch is shown in Table \ref{results:tab2}. 
We can see the number of parameters for our method is comparable with
SPD-SRU with higher test accuracy ($\approx 4\%$ improvement) and much faster runtime ($\approx 2.5\times$). 
Note that without residual connections, the accuracy drops to $0.809 \pm 0.044$: in other words,
residual connections are useful. 

{\bf Take-home message:} {\it With the above two experiments, we
can conclude that manifold DCNN \begin{inparaenum}[\bfseries (i)] \item is faster, \item uses fewer
parameters and \item gives better or comparable classification accuracy \end{inparaenum} compared to the state-of-the-art.}

\setlength{\intextsep}{0pt}%
\setlength{\columnsep}{5pt}%
\begin{table}[!b]
\setlength{\abovecaptionskip}{-0cm}
\setlength{\belowcaptionskip}{-0.55cm} 
   \centering
   \scalebox{0.75}{
\begin{tabular}{cccc} 
\topline\myrowcolour
{\bf Model} & {\bf \# params.} & {\bf time (s)/ epoch} & {\bf Test acc.}  \\
\midtopline
manifold DCNN & $3393$ & $\sim \highest{33}$ & $\highest{0.823 \pm 0.018}$ \\
SPD-SRU & $\highest{3337}$ & $\sim 76$ & $0.784 \pm 0.014$ \\
TT-GRU & $6048$ & $\sim 42$ & $0.78$ \\
TT-LSTM & $6176$ & $\sim\highest{33}$ & $0.78$ \\
SRU & $2535630$ & $\sim 50$ & $0.75$ \\
LSTM & $14626425$ & $\sim 57$ & $0.70$ \\
\bottomline
  \end{tabular}
}
\caption{\footnotesize Comparative results on UCF-11 data. Our model achieves the best accuracy and the fastest speed with a small number of parameters.}
\label{results:tab2}
\end{table}

\vspace{0.2cm}
\subsection{Group effects in Preclinical Alzheimer's disease}
\vspace{-0.1cm}
Cardinal features of Alzheimer's disease (AD) include the development of beta-amyloid plaques (amyloid), neurofibrillary tangles (tau), and progressive neurodegeneration (characterized by MRI) \cite{jack2018}. 
Autopsy studies among individuals with AD dementia indicate that degeneration of myelinated axons in the context of amyloid and tau pathology is a defining feature of dementia status \cite{perez2013dissecting}. Techniques for measuring axonal degeneration in vivo include analysis of cerebrospinal fluid, as well as diffusion-weighted imaging; however, few studies have tested the extent to which early amyloid accumulation may be associated with neural injury.
Our goal is to utilize our method to 
{\bf{identify white matter fiber bundles that are affected {\em early} in the preclinical disease process}}.
Positron emission tomography (PET) imaging with Pittsburgh compound B (PiB), which identifies amyloid deposition, 
can be used as an indicator of AD pathology \cite{ikonomovic2008post}.
Thus, we compared healthy individuals who were positive for AD pathology (PiB+)
to healthy individuals who were negative for pathology (PiB-).
Additionally, we compared individuals who carried a risk gene for AD 
(APOE+) to non-carriers (APOE-).

\vspace{-0.5cm}
\subsubsection{Diffusion-weighted imaging (DWI)}
\vspace{-5pt}
\vspace{-0.1cm}
{\bf Data acquisition:} Diffusion-weighted imaging was completed on a General Electric (GE) 3 Tesla scanner
with a $32$-channel head coil and a spin-echo echo-planar imaging pulse sequence among participants
who are asymptomatic. 
Multi-shell DWI data were collected using b-values $b=0,\ b=500,\ b=800,\ b=2000$,
with $2\times2\times2mm$ resolution. The signal was corrected using MRTrix3\cite{veraart2016denoising} 
and FSL's `eddy'\cite{andersson2016integrated}. 
Diffusion tensor imaging (DTI) and the orientation distribution functions (ODF), which were used as the representative of the DWI, 
were performed using the Diffusion Imaging in Python (DIPY) toolbox\cite{garyfallidis2014dipy}. 
To generate fiber bundles of interest, the data was processed using
TRACULA\cite{yendiki2011automated,yendiki2014spurious,yendiki2016joint}.
With this pipeline, we generated $18$ major fiber bundles \cite{wakana2007reproducibility}, as shown in Fig. \ref{DTI}(c).
Regions of interest (ROI) in the template space, 
were inversely warped back to the subject space to generate the fiber bundles and each data point used in the analysis for each participant. 

{\bf Analysis:} From the previous experiments, we can see that manifold DCNN performs well on classification problems with faster computation speed and fewer parameters. 
Due to the fast runtime {\em and} the small number of parameters, we can use permutation testing to perform
group analysis.
The statistical testing is performed on each fiber bundle between the two groups, to determine 
if the DCNN model between the two groups is different. 
To summarize, the setup is: 
\begin{inparaenum}[\bfseries (1)]
\item Group 1 (PiB+) versus Group 2 (PiB-),
\item Group 1 (APOE+) versus Group 2 (APOE-). 
\end{inparaenum}
Now, we will give some details of the DCNN models for DTI and ODF representations before the statistical analysis. 

\noindent\textbf{(i) Diffusion tensor imaging (DTI):}
Diffusion tensor imaging (DTI) is a method to represent the Diffusion imaging with SPD matrices. Since all of the data samples lie on the SPD manifold, the model is similar to the classification model above. 
The only difference between classification model and this group analysis model is that instead of the prediction of the classes, we are fitting the two groups of data into two trainable models, $\theta_1$ and $\theta_2$ and assessing if the distributions of $\theta_1$ and $\theta_2$ are statistically different.

\noindent\textbf{(ii) Orientation distribution function (ODF):}
Orientation distribution function (ODF) represents the probability densities of water diffusion over different orientations. 
In order to perform the statistical analysis, we discretized the space of orientations, i.e.,  $\mathbf{S}^{2}$. We sampled $724$ equally spaced points on the sphere $\mathbf{S}^{2}$ to represent the ODF. 
Let the ODF be denoted by $\mathbf{x}_t$, then after the discretization, we have $\sum_{i=1}^{724} \mathbf{x}_{t}^{i} = 1$. As ODF is a probability density function, we use square root parameterization \cite{brody1998statistical,srivastava2007riemannian} to represent ODF. Using the square root parameterization, we map $\mathbf{x}_t$ onto the positive orthant of the unit hypersphere of dimension $723$, i.e., $\mathbf{S}^{723}$. As in Section \ref{theory}, a key
component of DCNN is the definition of $\textsf{wFM}$, which we can define on $\mathbf{S}^n$:
\vspace*{-0.35cm}
\begin{align*}
\label{dcnnodf}
    \mathbf{y}(s) & = \textsf{wFM}\left(\left\{w(i)\right\},\left\{\mathbf{x}(s-d*(k-1):d:s)\right\}\right)\\[-1.5mm]
          & = \arg \min_{M}\sum_{i=0}^{k-1}w(i) d_{\mathbf{S}}^2\left(\mathbf{x}(s-d*i),M\right), \numberthis
\end{align*}
\vspace*{-0.45cm}

Here $d_{\mathbf{S}}$ is the rotation invariant geodesic distance on $\mathbf{S}^{723}$ and $\mathbf{x}(s)$ is a sample on $\mathbf{S}^{723}$ for $s \in \mathbf{N}$. Analogous to the SPD manifold, we can define a recursive $\textsf{wFM}$ estimator $\mathbf{m}_n$:
\vspace*{-0.4cm}
\begin{align}
\label{rodf}
    \mathbf{m}_0 = \mathbf{x}(s) \qquad \mathbf{m}_n = \Gamma_{\mathbf{m}_{n-1}}^{\mathbf{x}(s-n*d)}\left(\frac{w(n)}{\sum_{j=0}^{n}w(j)}\right),
\end{align}
\vspace*{-0.4cm}

where $\Gamma$ is the shortest geodesic on $\mathbf{S}^{723}$. Using the above-defined estimator of $\textsf{wFM}$, we can define DCNN on $\mathbf{S}^{723}$ as in Section \ref{theory}.
{\bf{Note}}: Our baseline model, SPD-SRU cannot deal with the $S^n$ manifold as we do here.


\vspace{-0.3cm}
\subsubsection{Statistical analysis: permutation testing}
\vspace{-0.1cm}
Suppose we train our model for each of the two groups for each fiber bundle $^{fb}$ we have,
with parameters $\theta_{1}^{fb}$ and $\theta_{2}^{fb}$. 
Our goal is to test whether the fiber bundle $^{fb}$ is statistically different between the two groups.
Thus, we model the manifold-valued data and perform statistical analysis in the parameters space.
Since the models for each group lie in the same parameter space,
the statistical analysis can be performed in the parameter space by bootstrapping.
We can measure the distance between two models as $\sigma^{fb} = ||\theta_{1}^{fb}-\theta_{2}^{fb}||$ to represent the distance between the group-wise fitted models' distributions in parameter space. 
Then, we need to evaluate how statistically significant the distance is -- and if
the value is large enough, it is unlikely to happen by chance. 
A simple way to perform the test for statistical significance is via permutation testing. 
If we randomly shuffle (via a random permutation) the group information for all our samples (i.e., subjects) and
run our model for both ``random'' groups,
we will get new parameters $\hat{\theta}_{1}^{fb}$ and $\hat{\theta}_{2}^{fb}$.
We define $\hat{\sigma}^{fb} = ||\hat{\theta}_{1}^{fb}-\hat{\theta}_{2}^{fb}||$ as a random variable. 
After permuting 5000 times, we can estimate the distribution of the $\hat{\sigma}^{fb}$ -- this is
the Null distribution (See Fig. \ref{ptest} as examples). 
The $p$-value is defined as the ranking of the $\sigma^{fb}$ among the distribution of the $\hat{\sigma}^{fb}$. 
If the $p$-value is less than the significance threshold $\alpha = 0.05$,
we can conclude that this is {\bf{not}} likely to happen by chance.

Since the length of different fiber bundles varies from $11$ to $73$, we construct the DCNN to
have $3$ layers of residual units, 
with channels being $1,3,3$; $3,3,5$ and $5,8,10$ respectively. 
And the $1$-D kernel size is $3$. We use all the data we have to pre-train the model. 
After pre-training, we fine tune the model during the permutation testing. 

\begin{table}[!t]
\label{results:tab0}
\vspace*{-0.0cm}
\setlength{\abovecaptionskip}{-0cm}
\setlength{\belowcaptionskip}{-0.15cm} 
   \centering
   \scalebox{0.5}{
\begin{tabular}{ccccccc} 
\topline\myrowcolour
 & \multicolumn{3}{c}{{\bf PiB}} & \multicolumn{3}{c}{{\bf APOE}} \\
\myrowcolour
\multirow{-2}{*}{\bf Experiments} &  {\bf Total}    & {\bf Positive} & {\bf Negative}  &  {\bf Total} & {\bf Positive} & {\bf Negative} \\
\midtopline
Number                            & 196             &  29            & 167             & 669          & 247            & 422 \\
Age (years ) (mean (SD))          & 62.40 (6.33)    &  66.29 (4.95)  & 61.75 (6.30)    & 65.61 (8.68) & 64.55 (7.99)   & 66.23 (9.00)\\
Sex  (female; \%)                 & 134 (68\%)      &  21 (72\%)     & 113 (68\%)      & 426 (64\%)   & 159 (64\%)     & 267 (63\%) \\ 
\bottomline
\end{tabular}
}
\caption{\footnotesize \label{demographic}Description of data/participant demographics used in the study.}

\end{table}

\vspace{-0.5cm}
\subsubsection{Result 1: Group analysis: PiB+ versus PiB-}
\vspace{-0.2cm}{}
The study included imaging data acquired from $196$ cognitively unimpaired (healthy) participants acquired in a local cohort at the University of Wisconsin. We provide demographic information from participants with PiB and APOE measures in
Table \ref{demographic}. Initial analyses were run using single-shell data, where the model was run on all $18$ fiber bundles, one by one, with the parameters mentioned above.
We performed permutation tests for each fiber bundle individually. 

Results for the $18$ fibers are shown in Table \ref{results:tab3} (column 2). 
We find that two of the $18$ fibers satisfied the threshold of $0.05$, which means that statistically these fiber bundles are different across the two groups. 
Since the sample sizes were small, the results presented are uncorrected $p$-values (multiple testing correction was not performed). 

Fiber bundles evaluated in this analysis included those which are known {\bf{to be affected}} in AD,
including the superior longitudinal fasciculus and cingulum bundle, as well
as control tracts that are {\bf{not likely to be affected}} by AD, such as the corticospinal tract. 
We found significant differences between PiB+ and PiB- groups in fiber bundles that are
likely to be affected by AD, including the superior longitudinal fasciculus and 
Corpus callosum - forceps minor. 

When compared with the SPD-SRU model, which also reported brain imaging experiments in their paper, the results show only one out of $18$ fibers survives. 
And also, we find that our model runs {\bf{much faster}} (about 5$\times$), which is very important when running permutation testing thousands of times. 
It takes 3.5 days to run permutation testing 5000 times using DCNN, while the SPD-SRU takes 18 days.
When we keep the number of GPUs fixed, the difference between 3.5 and 18 will be even more sizable if we expand the number of permutation testing to 10000 or more.

\vspace{-0.5cm}
\subsubsection{Result 2: Group analysis: APOE+ versus APOE-}
\vspace{-0.1cm}
The APOE analysis was performed using data from $669$ subjects with APOE information, with $247$ of them being positive for APOE4 (a risk factor for AD). 
Analyses were also conducted using the multi-shell dMRI to generate ODF information. 
Similar to the preceding group difference analysis, the model was run on all
$18$ fiber bundles with the parameters described previously on both DTI and ODF. 

The results for $18$ fibers are shown in Table \ref{results:tab3} in column 3. 
It is noteworthy that SPD-SRU can only deal with the SPD manifold. 
So for ODF, which lies on $S^n$, we can {\it{only}} run our DCNN model to do the group analysis.

Here, we found that four of the $18$ fiber bundles met the significance threshold of $0.05$ with DTI, while SPD-SRU only captured one. 
Five fiber bundles were identified when using ODF. 
We found differences by APOE genotype in the forceps minor, cingulum projecting to parietal cortex,
anterior thalamic projections, superior longitudinal fasciculus projecting to parietal cortex and inferior longitudinal fasciculus.
We did not find differences in fiber bundles
unlikely to be affected by AD, such as the corticospinal tract in both experiments. 
Fiber bundles that were consistently identified in both the DTI and ODF analyses included the inferior longitudinal fasciculus and the anterior thalamic projections. 

\vspace{-0.5cm}
\subsubsection{Discussion of preclinical AD analysis results}
\vspace{-0.2cm}

\begin{figure}[!t]
\setlength{\abovecaptionskip}{-0cm}
\setlength{\belowcaptionskip}{-0cm} 
        \centering
                \includegraphics[scale=0.1072]{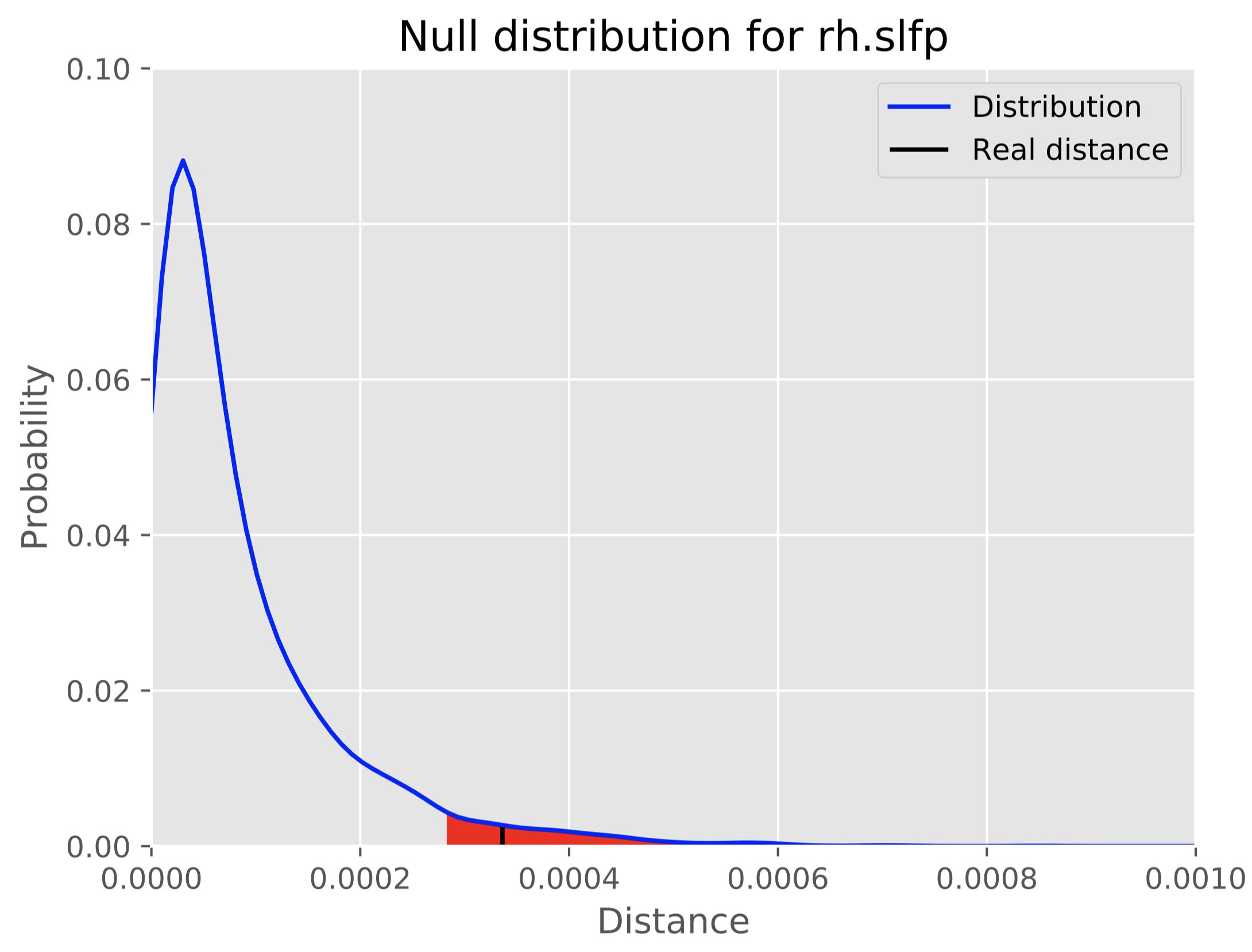}
                \includegraphics[scale=0.1072]{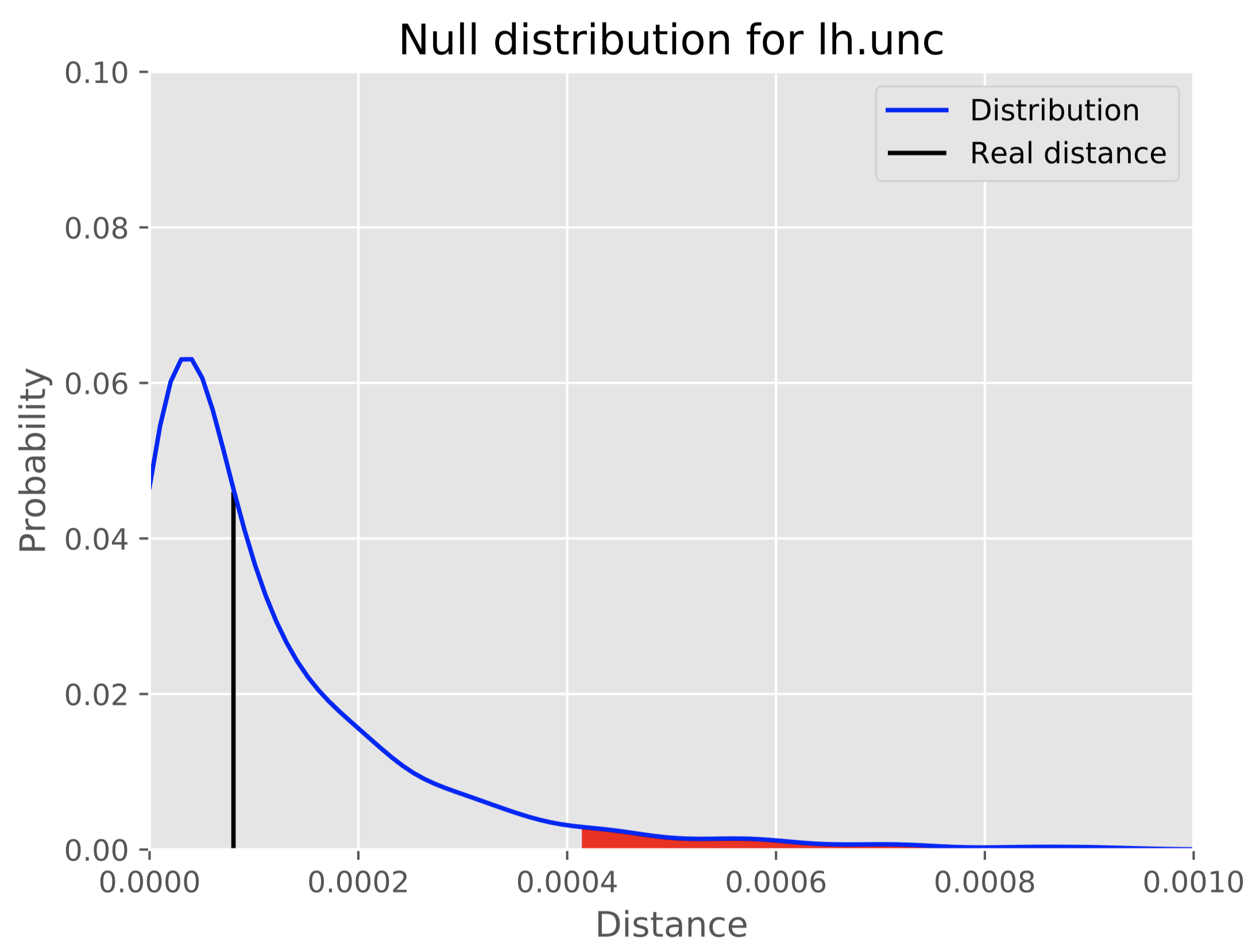}
                \caption{\footnotesize The Null distribution for one fiber bundle with $\alpha = 0.05$. If the real distance (black line) lies in the threshold (red area), that test is believed to not happen by chance.}
                \label{ptest}
\vspace*{-0.5cm}
\end{figure}
While amyloid and tau pathology are defining features of AD, methods are also needed to detect AD-associated neurodegeneration \cite{jack2018nia}. 
Neurodegeneration may signal future cognitive decline. However, methods for detecting early and subtle neurodegeneration, particularly of myelinated axons, are not yet available, especially in preclinical AD.
This is why our results here seem promising.

The results suggest significant differences in underlying fiber bundle microstructure among individuals who meet biological criteria for AD 
(based on PiB status) as well as differences by APOE genotype. 
Of note, our algorithm identified significant differences in the cingulum bundle by PiB status; 
this white matter fiber bundle connects medial temporal lobe and parietal cortices as part of a memory network that is impacted by AD, 
and is vulnerable to degeneration in the early stages of AD. 
Differences in the cingulum bundle were also apparent among carriers of the APOE4 allele, 
a genetic risk factor for sporadic AD. 
Likewise, superior longitudinal fasciculus differed by AD biomarker status and APOE genotype. 
Projections identified as being significantly different included fiber bundles projecting to parietal cortices. 
Parietal cortices are significantly impacted by AD pathology and are among the first to show amyloid accumulation. 
The results presented here may suggest that amyloid accumulation negatively impacts adjacent white matter fiber bundles. 
It may also be possible that degeneration of fiber bundles is a function of AD pathology spreading to anatomically linked brain regions via white matter fiber bundles, 
although further longitudinal evaluation is needed to test the hypothesis. 
In summary, statistical analysis enabled by our proposed algorithm was
capable of identifying differences in biologically meaningful brain regions. 

{\bf Take-home message:} {\it Our DCNN model was able to capture more fiber differences with significant effects compared to the SPD-SRU. 
It is also noteworthy that our model is much more efficient: only $60$s for one realization of the permutation test ($\times \#$ of realizations), 
while the SPD-SRU model $>5\times$ times slower. 
Compared with the SPD-SRU, which can only handle DTI (SPD), our method is more general: handles
both DTI (SPD) and ODF ($\mathbf{S}^n$) data.} 

\setlength{\intextsep}{0pt}%
\setlength{\columnsep}{0pt}%
\begin{table}[!t]
\setlength{\abovecaptionskip}{-0.0cm}
\setlength{\belowcaptionskip}{-0.25cm} 
\vspace{-0.2cm}
   \centering
   \scalebox{0.64}{
\begin{threeparttable}
\begin{tabular}{cccccc} 
\topline\myrowcolour
                                & \multicolumn{5}{c}{$p$-value} \\
\myrowcolour
{\bf Fiber Name} & \multicolumn{2}{c}{\bf Experiment 1}                & \multicolumn{3}{c}{\bf Experiment 2} \\
\myrowcolour
              & \multicolumn{2}{c}{PiB+ versus PiB-} & \multicolumn{3}{c}{APOE+ versus APOE-}\\
\myrowcolour
              & {DCNN}            & {SPD-SRU}          & {DCNN}              & {SPD-SRU}          & {DCNN}              \\
\myrowcolour                                                       
              & {on DTI}          & {}                 & {on DTI}            & {}                 & {on ODF}            \\
\midtopline                                                       
fmajor\_PP    & 0.443            & 0.923              & 0.207               & 0.600              & 0.778               \\
fminor\_PP    & $\highest{0.008}$& 0.158              & $\highest{0.035}$   & $\highest{0.025}$  & N/A                 \\
\myrowcolour                                                       
lh.atr\_PP    & 0.323            & 0.632              & 0.30                & 0.991              & $\highest{0.028}$   \\
\myrowcolour                                                       
rh.atr\_PP    & 0.295           & 0.143              & 0.86                & 0.271              & 0.563               \\
lh.cab\_PP    & 0.276           & 0.363              & 0.76                & 0.644              & 0.500               \\
rh.cab\_PP    & 0.311           & 0.263              & 0.78                & 0.848              & 0.444               \\
\myrowcolour                                                       
lh.ccg\_PP    & 0.230            & 0.267              & $\highest{0.042}$   & 0.609              & $\highest{0.043}$   \\
\myrowcolour                                                       
rh.ccg\_PP    & 0.093           & 0.087              & $\highest{0.048}$   & 0.532              & $\highest{0.048}$   \\
lh.cst\_AS    & 0.561            & 0.143              & 0.58                & 0.350              & 0.800               \\
rh.cst\_AS    & 0.629           & 0.278              & 0.35                & 0.667              & 0.769               \\
\myrowcolour                                                       
lh.ilf\_AS    & 0.309            & 0.895              & 0.47                & 0.977              & $\highest{0.042}$   \\
\myrowcolour                                                       
rh.ilf\_AS    & 0.405           & 0.889              & 0.46                & 0.563              & 0.857               \\
lh.slfp\_PP   & 0.482            & 0.615              & 0.68                & 0.107              & 0.192               \\
rh.slfp\_PP   & 0.571           & 0.941              & $\highest{0.047}$   & 0.154              & $\highest{0.050}$   \\
\myrowcolour                                                       
lh.slft\_PP   & $\highest{0.005}$& $\highest{0.041}$  & 0.92                & 0.649              & 0.556               \\
\myrowcolour                                                       
rh.slft\_PP   & 0.790           & 0.462              & 0.53                & 0.947              & 0.333               \\
lh.unc\_AS    & 0.623            & 0.158              & 0.23                & 0.860              & 0.933               \\
rh.unc\_AS    & 0.298           & 0.895              & 0.34                & 0.324              & 0.182               \\
\bottomline                                                                          
\end{tabular}
\begin{tablenotes}
        \footnotesize
        \item[*]N/A: This ODF fiber bundle did not pass Quality Check (QC) after pre-processing. Therefore, we left it out of the analysis to avoid inconsistencies in the parameters used for pre-processing the full set of fiber bundles. 
\end{tablenotes}
\end{threeparttable}
}

   \caption{\footnotesize $p$-values (uncorrected) for all fibers in different groups. The highlights are the fiber bundles that satisfy the significance threshold.
     Runtime for DCNN is $5\times$ times faster than SPD-SRU (not included here).}
\label{results:tab3}

\end{table}

\section{Conclusions}
We present a new Dilated CNN formulation to model sequential and spatio-temporal manifold data, where
few alternatives are available. 
Compared with the standard sequential model (RNN), our method can
improve the performance when evaluated on the number of parameters and runtime.
We show that when using wFM, Weight normalization, ReLU, and Dropout are no longer needed in this formulation.
On the experimental side, for video analysis, we show that improvements can be obtained
with fewer parameters and shorter running time.
Importantly, we show that our algorithmic contributions facilitate scientific discovery relevant to AD,
and may facilitate early disease detection at the preclinical stage.
The analysis enabled by our formulation revealed subtle neurodegeneration of white matter fiber bundles affected by AD pathology,
in brain regions implicated in prior studies of AD.
The code is available at https://github.com/zhenxingjian/DCNN.

\section*{Acknowledgments}

This research was supported in part by grants R01EB022883, R01AG059312, RF1AG027161 (WRAP study), R01AG021155,  P50AG033514 (Wisconsin ADRC), R01 AG037639,
P30 AG062715, 
R01 AG059312, 
RF1 AG027161,   
and NSF CAREER award RI 1252725. 

\newpage
{\small
\bibliographystyle{ieee_fullname}
\bibliography{egbib}
}

\end{document}